\documentclass{article}
\usepackage[nonatbib, preprint]{neurips_2020}

\usepackage[utf8]{inputenc} 
\usepackage[T1]{fontenc}    
\usepackage{hyperref}       
\usepackage{url}            
\usepackage{booktabs}       
\usepackage{amsfonts}       
\usepackage{nicefrac}       
\usepackage{microtype}      

\usepackage{amsmath}		
\usepackage{amsthm}			
\usepackage{amssymb}		
\usepackage{multicol}
\usepackage{graphicx}	
\usepackage{subfig}
\usepackage{mathtools}
\usepackage{bbm}
\usepackage{braket}
\usepackage{algorithm2e}
\usepackage{mdframed}
\usepackage[font=small,labelfont=bf]{caption}
\usepackage{cite}

\newcommand{\dtv}{\bar{\delta}_{tv}}

\newtheorem{theorem}{Theorem}

\newtheorem{lemma}{Lemma}
\newtheorem{defn}{Definition}

\newmdtheoremenv{intuit}{Insight}
\newcommand{\mc}[1]{\mathcal{#1}}

\usepackage[many]{tcolorbox}

\tcolorboxenvironment{proof}{
  blanker,
  before skip=\topsep,
  after skip=\topsep,
  borderline west={0.4pt}{0.4pt}{black},
  breakable,
  left=10pt,
  right=12pt,
}

\DeclarePairedDelimiterX{\inp}[2]{\langle}{\rangle}{#1, #2}

\begin{document}

\title{On the Maximum Mutual Information Capacity of Neural Architectures}

\author{Brandon~Foggo and~Nanpeng~Yu,
}
    
\tikzset{
		block/.style = {draw, fill=white, rectangle, minimum height=1em, minimum width=3em},
		tmp/.style  = {coordinate}, 
		sum/.style= {draw, fill=white, circle, node distance=1cm},
		input/.style = {coordinate},
		output/.style= {coordinate},
		pinstyle/.style = {pin edge={to-,thin,black}
		}
	}
     
\RestyleAlgo{boxruled}

\maketitle
    
\begin{abstract}
We derive the closed-form expression of the maximum mutual information - the maximum value of $I(X;Z)$ obtainable via training - for a broad family of neural network architectures. The quantity is essential to several branches of machine learning theory and practice. Quantitatively, we show that the maximum mutual information for these families all stem from generalizations of a single catch-all formula. Qualitatively, we show that the maximum mutual information of an architecture is most strongly influenced by the width of the smallest layer of the network - the ``information bottleneck'' in a different sense of the phrase, and by any statistical invariances captured by the architecture.
\end{abstract}


\section{Introduction}
Information Theory acts as a powerful complement to traditional statistical machine learning theory. It provides bounds that are agnostic to parameter size, yields interpret-able results, and is physically related to several fields of application. However, the rigorous study of information-theoretic properties relating to the architecture of neural networks remains somewhat immature.  In this paper, we will study, for a large family of architectures, the maximum mutual information (MMI) that the network's hidden representation $Z$ can hold from a feature space $\mc{X}$ with distribution $\mathbb{P}_{X}$.  

That is, we are studying the supremum $\sup_{\theta} I(X;Z_{\theta})$ where $\theta$ is the parameter space of our architecture, $Z$ is the random variable denoting the output of our final hidden layer (our representation), and $I(X;Z)$ is the mutual information between the representation and the input. The reason for studying this quantity is due to the prevalence of the following concept in Information-Theoretic Machine Learning Theory (ITMLT): $I(X;Z)$ acts as a data-sensitive measure of complexity for the representation. Indeed, we can often obtain bounds on machine learning performance in a way that is highly dependent on $I(X;Z)$. In this sense, $\sup_{\theta} I(X;Z_{\theta})$ is a data-sensitive measure of complexity for the network itself. Thus its study can give us insight into why certain architectures perform well on certain datasets and may be useful for finding architectures that will work well with new datasets.

In a much more practical sense, Information Theoretic Machine Learning Theory establishes the existence of a ``best'' value of $I(X;Z)$, which maximizes the potential of obtaining a good representation. While this target always exists, it typically will not be known entirely. A sub-branch of this field argues that neural networks will attempt to find this target naturally through training \cite{shwartzziv2017opening, tishby2015deep, achille2018emergence}, but there is no guarantee that the training process will find it. It is, however, often possible to approximate this target with an upper bound. We can then choose our architecture to enforce this upper bound naturally via its maximum mutual information - significantly reducing the search space for the target $I(X;Z)$ in the training process.

While the capacity of neural networks has been analytically studied in a different sense [i.e. via storage of 'patterns'] \cite{mackay2003information} (chapter 5), \cite{prados1989neural}, and numerical methods of estimation of mutual information in deep networks exist \cite{gabrie2018entropy, paninski2003estimation, bahl1986maximum, pmlr-v80-belghazi18a, hjelm2018learning, gao2017estimating}, no analytical studies of this type exist in literature. 


\section{Notation, Background, and Motivation}\label{sec:background}
Our primary motivation for studying maximum mutual information comes from some modern theoretical work in ITMLT, which studies the classification of a discrete variable ${Y \in \mc{Y}}$ from an input $X \in \mc{X}$, jointly distributed as $p(x,y)$. The theory studies the potential losses in the quality of a learning machine's representation of its input, $Z$, when trained from just a sample of data $S = \left((X_1, Y_1), \cdots, (X_m, Y_m)\right)$. 

A natural quality measure for a representational variable is its mutual information with the classification variable, $I(Y;Z)$ \cite{cover2012elements}. This is a measure of how well $Z$ predicts the class variable $Y$, and is a component in the objective functions of several deep learning methods \cite{chen2016infogan, foggo2019improving, alemi2016deep, kolchinsky2019nonlinear, bang2019explaining}. We can then conceive of a ``best possible'' representation being that which maximizes $I(Y;Z)$ when the natural distribution $p(x,y)$ is known perfectly. This ``best possible'' representation is denoted $Z^*$. Of course, we typically don't know $p(x,y)$ perfectly, but we can still obtain a representation $\hat{Z}(S)$ from our samples. If we constrain both variables, $Z^*$ and $\hat{Z}$, to have a fixed mutual information with $X$ (i.e. $I(X;Z^*) = I(X;\hat{Z}) \triangleq I(X;Z)$), then we then have the following bound on the quality loss of $\hat{Z}(S)$ against $Z^*$ \cite{Foggo_2019}: 
\begin{equation}\label{eqn:prod_form}
    \underbrace{\left|I(Y;\hat{Z}(S)) - I(Y;Z^*) \right|}_{\text{quality loss}} \leq 2\underbrace{\dtv}_{\text{data}}(S)\underbrace{I(X;Z)}_{\text{architecture}} + 2h_2(\dtv(S))
\end{equation}

Note that the constraint $I(X;Z^*) = I(X;\hat{Z}) \triangleq I(X;Z)$ is a traditional choice first made in the early literature on the Information Bottleneck method \cite{tishby2000information, slonim2000agglomerative, friedman2013multivariate} - a method which helped spark the existence of ITMLT. The constraint just exists to normalize the two representations to a fixed level of "forgetfullness" from the input. The under-braces on the right hand side of this inequality - implying that $\dtv$ is data dependent and $I(X;Z)$ is architecture dependent in a decoupled way - is justified by the work introducing that bound \cite{Foggo_2019}. In that work, $\dtv$ was upper bounded in a way that did not depend on the architecture's complexity by terms on the order of $O\left(\sqrt{\frac{|\mc{Y}|m'log(2)}{2m}} \right)$ where $m'$ is a small integer, and $m$ is the size of the training sample. Further research has verified similar data dependent bounds \cite{br2019analyzing}. Since these bound did not depend on architectural complexity, all of the dependence on architecture must be contained in the other term: $I(X;Z)$. 

On the other side of the coin, we have strong data processing inequalities \cite{polyanskiy2017strong, polyanskiy2015dissipation} such as $I(Y; Z^*) \leq \eta_{X,Y} I(X;Z)$ where $\eta_{X,Y} \leq 1$ is the maximum correlation coefficient of $X$ and $Y$, i.e. $\eta_{X,Y} = \underset{f,g}{\sup~}\mathbb{E}\left[f(X)g(Y)\right]$, the supremum being taken over all functions $f: \mc{X} \to \mathbb{R}$ and all functions $g: \mc{Y} \to \mathbb{R}$ with $\mathbb{E}[f] = \mathbb{E}[g] = 0$, and ${\|f\|_{2} = \|g\|_{2} = 1}$. 

A qualitative combination of these two contrasting inequalities is summarized as follows: $I(Y; Z^*)$ increases quickly with $I(X;Z)$ up to its maximum of $I(Y;X)$ [Strong Data Processing Inequality], and $I(Y;\hat{Z})$ follows this trend due to the low information losses at low $I(X;Z)$. But eventually the risk of information loss by inequality (\ref{eqn:prod_form}) forces our best estimate of $I(Y; \hat{Z})$ to decrease linearly from this maximum value. The location of this behavioral change is problem-specific. But since $\eta_{X,Y}$ can be estimated via a small sample of data, these equations can be used to approximate a target $I(X;Z)$ - yielding an interval in which it is likely to be contained. Setting our architecture to have a maximum mutual information near the supremum of this estimated interval is desirable.


\section{The MMI Bottleneck and Parallel Structures}
We begin with a key takeaway in terms of series and parallel components. The MMI over channels in series is the smallest MMI of the series, and the MMI over channels in parallel is the sum of the parallel MMI values. One consequence of this takeaway is that the information theoretic properties of fully connected architectures are strongly dominated by the dimension of the smallest layer.

We will often be able to identify a dominant structural parameter that limits the MMI of a series calculation. When we identify the parameter, we will call it the \emph{MMI bottleneck}.  In a fully connected network, this will be given by the dimension of the smallest hidden layer. Special focus should be given to the MMI bottleneck when attempting to control $I(X;Z)$.


\section{Single Layer Linear Networks}
We will first study networks consisting of a single layer, and with no activation function. While this is a highly specialized case, the results and methods of obtaining those results generalize quite well to other cases. This section will consider both fully connected architectures and convolutional architectures.

\subsection{Fully Connected Case}
We begin by deriving the MMI of a linear network with a standardized Gaussian input. This is a highly specialized case, but we will see that it generalizes quite nicely to a large family of architectures including convolutional architectures and architectures with relu activation functions.  

We consider the constrained problem in which the weight matrices $W$ are constrained by Frobenius norm. We will see that the Maximum Mutual Information of this family of architectures is discontinuous in the Frobenius norm constraint with at most ${dim(\mc{X})}$ points of discontinuity. We will thus first specialize to the case where our Frobenius norm constraint is larger than the largest discontinuity point before moving to the more general scenarios. 


\begin{theorem}\label{thm:lingaussmmi}
Let $\Sigma_x$ be a positive definite matrix and let $\sigma^2 > 0$. Let $N_0$ and $N_1$ be natural numbers representing the input and hidden dimensions. Let $\mc{N}(\mu; A)$ denote the Gaussian distribution with mean $\mu$ and covariance matrix $A$. Let $ X \sim \mc{N}(0; \Sigma_x)$, $X \in \mathbb{R}^{N_0}$, $Z|X,W,b \sim \mc{N}(WX+b; ~\sigma^2 Id_{N_1})$, $Z \in \mathbb{R}^{N_1}$, where $W \in \mathbb{R}^{N_1 \times N_0}$, $Id_{N_1}$ is the identity matrix in $N_1$ dimensions, and $b \in \mathbb{R}^{N_1}$ is the bias vector. Let ${\tilde{N}=\min(N_0, N_1)}$. Let $\Sigma_{x,\tilde{N}}$ denote $\tilde{N} \times \tilde{N}$ diagonal matrix containing the $\tilde{N}$ largest eigenvalues of $\Sigma_x$. Let $\lambda^x_{\tilde{N}}$ denote the smallest eigenvalue of $\Sigma_{x,\Tilde{N}}$, and let $\rho_{\tilde{N}} \triangleq \sigma^2 \left(\frac{\tilde{N}}{\lambda^x_{\tilde{N}}} - Tr(\Sigma_{x,\tilde{N}}^{-1})\right)$. Let $F \geq \rho_{\tilde{N}}$, $I_{W,b}(X;Z) \triangleq \mathbb{E}_{\sim p(x,z|W, b)}\left[log \frac{p(z|x,W, b)}{p(z|W, b)} \right]$, and define $\text{MMI}(X;Z) \triangleq \underset{Tr(W^TW) \leq F}{\text{ sup }} I_{W,b}(X;Z)$. Then: 
\begin{equation}
    \text{MMI}(X;Z) = \frac{\tilde{N}}{2}log \left(\frac{F + \sigma^2Tr(\Sigma_{x,\tilde{N}}^{-1})}{\sigma^2 \tilde{N}} \right) + \frac{1}{2}log~|\Sigma_{x,\tilde{N}}|
\end{equation}
\end{theorem}

\begin{proof}~

Since $X$ is Gaussian and the network is linear, $Z$ is Gaussian for all $W, b$. Thus, we can express $I_{W,b}(X;Z)$ as $\frac{1}{2}log~\frac{|\Sigma_z|}{|\Sigma_{z|x}|} = \frac{1}{2}log~\frac{|\sigma^2 Id_{N_1} + W\Sigma_x W^T|}{|\sigma^2 Id_{N_1}|}
    = \frac{1}{2}log~|Id_{N_1} + \frac{1}{\sigma^2}W\Sigma_x W^T|$. Now, by the matrix determinant lemma, we have that $|Id_{N_1} + \frac{1}{\sigma^2}W\Sigma_x W^T| = {|\sigma^2 \Sigma_x^{-1} + W^TW| \cdot |\frac{1}{\sigma^2}\Sigma_x|}$, and so we can condense the dependence of the MMI optimization problem on $W$ to obtain:
\begin{align}\label{eqn:opt_fullform}
    &\text{MMI}(X;Z) = \frac{1}{2}log~|\frac{1}{\sigma^2}\Sigma_x|  + \underset{Tr(W^TW) \leq F}{\sup~} \frac{1}{2} log~|Q(W)|
\end{align}
where ${Q(W) = \sigma^2 \Sigma_x^{-1} + W^TW}$. Due to the positive definiteness of $Q$ and Hadamard's inequality, we can cast this constrained maximization problem into the realm of eigenvalues since the optimal $Q$ matrix will be diagonal, so $W^TW$ will have the same eigenvectors as $\Sigma_x^{-1}$:
\begin{align}\label{prob:MMIgl_eigform}
    \underset{\tilde{\lambda}_1, \cdots, \tilde{\lambda}_{N_0}}{\text{ sup }}\sum_{i=1}^{N_0} log \left(\tilde{\lambda}_i + \frac{\sigma^2}{\lambda^x_{i}}\right), ~
    \text{ s.t. }  \sum_{i=1}^{N_0}\tilde{\lambda}_i \leq F,~ \tilde{\lambda}_i \geq 0, ~i=1,2,\cdots,N_0 \notag \\
    \text{ and } \tilde{\lambda}_i = 0, \text{ for at least } \max(0,N_0-N_1) \text{ values of } i
\end{align}
where ${\lambda^x_i}$ is the $i^{th}$ largest eigenvalue of $\Sigma_x$, and ${\tilde{\lambda}_i}$ is the $i^{th}$ un-ordered eigenvalue of $W^TW$. The final constraint comes from the fact that $W^TW$ is only rank $\min(N_0, N_1)$. Furthermore, it must be the case that the mandatory $0$-valued eigenvalues of $W^TW$, when they exist (${N_0 > N_1}$), must be placed on the indices $i=N_1+1, \cdots, N_0$, as these correspond to the largest values of $\frac{\sigma^2}{\lambda_i^x}$. Indeed, suppose that we have placed a nonzero eigenvalue on one of these indices (without loss of generality, say index $p$, and that we set this eigenvalue to $l$) in such a way that all of the constraints are met. Then we must have placed a zero-valued eigenvalue on another index (which we will denote as $q$ without loss of generality, $q<p$). Then the objective function can be increased without violating any constraints by taking $l$ units of eigenvalue off of index $p$ and placing it on index $q$, and so this cannot be a solution to our optimization problem. To see this, observe that:
\begin{align}
    log\left(l+\frac{\sigma^2}{\lambda^x_p}\right) + log \left(\frac{\sigma^2}{\lambda^x_q}\right) = log\left(\frac{l\sigma^2}{\lambda^x_q} + \frac{\sigma^4}{\lambda^x_p\lambda^x_q}\right)
    &\leq log\left(\frac{l\sigma^2}{\lambda^x_p} + \frac{\sigma^4}{\lambda^x_p\lambda^x_q}\right) \notag \\
    &= log\left(\frac{\sigma^2}{\lambda^x_p}\right) + log \left(l + \frac{\sigma^2}{\lambda^x_q}\right)
\end{align}
 We are thus left with the following optimization problem: 
\begin{gather}\label{prob:MMIgl_eigform_simp}
    \underset{\tilde{\lambda}_1, \cdots, \tilde{\lambda}_{\tilde{N}}}{\text{ sup }}\sum_{i=1}^{\tilde{N}} log \left(\tilde{\lambda}_i + \frac{\sigma^2}{\lambda^x_{i}}\right), ~ \text{ s.t. }  \sum_{i=1}^{\tilde{N}}\tilde{\lambda}_i \leq F, ~\tilde{\lambda}_i \geq 0, ~i=1,2,\cdots,\tilde{N}
\end{gather}
This is a classic `water-filling' problem with heights given by scaled versions of the inverses of the first $\tilde{N}$ eigenvalues of $\Sigma_x$. Thus, for a given `water level', $\mu^*(F)$, a solution is readily available, being given by $\tilde{\lambda}_i = \text{max}\left(0, \mu^*(F) - \frac{\sigma^2}{\lambda^x_{i}}\right)$ [optimality]. However, finding the relationship between $\mu^*(F)$ and $F$ requires additional work, as $\mu^*(F)$ must satisfy $\sum_i \text{max}\left(0, \mu^*(F) -  \frac{\sigma^2}{\lambda^x_{i}}\right) = F$ [consistency]. We will show that our assumption, ${F \geq \rho_{\tilde{N}}}$, yields a consistent solution in which all maximums of the optimality equation are obtained in the second argument. To see this, note that under such a solution, the consistency equation yields $\mu^*(F) = \frac{F + \sigma^2Tr(\Sigma_{x,\tilde{N}}^{-1})}{\tilde{N}}$, which coincides with the optimality equation because, for each $i$, $\frac{F + \sigma^2Tr(\Sigma_{x,\tilde{N}}^{-1})}{\tilde{N}} - \frac{\sigma^2}{\lambda^x_{i}} \geq \frac{1}{\tilde{N}}\left(F - \rho_{\tilde{N}} \right) \geq 0$. Thus this solution holds and, in all, we have an MMI of:
\begin{align}
& \frac{1}{2}log|\frac{\Sigma_x}{\sigma^2}| + \frac{\tilde{N}}{2}log \left(\frac{F + \sigma^2Tr(\Sigma_{x,\tilde{N}}^{-1})}{\tilde{N}} \right) + \frac{1}{2}\sum_{i=\tilde{N}+1}^{N_0}log(\frac{\sigma^2}{\lambda^x_i})   \notag \\
= & \frac{1}{2}\sum_{i=1}^{\tilde{N}}log(\frac{\lambda^x_i}{\sigma^2}) + \frac{\tilde{N}}{2}log \left(\frac{F + \sigma^2Tr(\Sigma_{x,\tilde{N}}^{-1})}{\tilde{N}} \right)
= \frac{1}{2}log|\Sigma_{x,\tilde{N}}| + \frac{\tilde{N}}{2}log \left(\frac{F + \sigma^2Tr(\Sigma_{x,\tilde{N}}^{-1})}{\sigma^2\tilde{N}} \right)
\end{align}
completing the proof.
\end{proof}

Some takeaways from Theorem \ref{thm:lingaussmmi} are now in order. First, there is an MMI bottleneck given by $\min \{N_0, N_1 \}$ - the minimum of the input dimension and the hidden dimension. Secondly, if ${N_1 < N_0}$, then $N_1$ controls the number of principal components used to maximize the mutual information. Thirdly, the largest discontinuity point of $\text{MMI}(F)$ is dominated by the difference between the \emph{largest} reciprocal eigenvalue and the \emph{average} reciprocal eigenvalue of those principal components that are used. We also see that smaller principal components are removed first. Furthermore, if $\tilde{N}$ is large, then we have the following approximation $\text{MMI}(X;Z) \approx \frac{1}{2}\sum_{i=1}^{\tilde{N}}log(\lambda^{x}_{i} \cdot \bar{\lambda}^{-1})$ where $\bar{\lambda}^{-1}$ is the average reciprocal eigenvalue of those components that are used. 

We now move on to defining the rest of the discontinuity points. They are defined in the following Lemma.


\begin{lemma}\label{lemma:ordered_rho}
Take all of the assumptions from Theorem \ref{thm:lingaussmmi} except for the assumption that ${F \geq \rho}$. Let the $k^{th}$ largest eigenvalue of $\Sigma_{x}$ be denoted by $\lambda^x_k$. Let $K$ be a natural number, ${K < \tilde{N}}$, and let $\Sigma_{x,\tilde{N}-K}$ denote the ${(\tilde{N}-K) \times (\tilde{N}-K)}$ diagonal matrix containing the $\tilde{N}-K$ largest eigenvalues of $\Sigma_x$. Now, let $\rho_{\tilde{N}-K} \triangleq \sigma^2 \left(\frac{\tilde{N}-K}{\lambda^x_{\tilde{N}-K}} - Tr\left(\Sigma_{x,\tilde{N}-K}^{-1}\right)\right)$. Then:
\begin{equation}
    0 = \rho_{1} \leq \cdots \leq \rho_{\tilde{N}-K} \leq \rho_{\tilde{N}-K+1} \leq \cdots \leq \rho_{\tilde{N} - 1} \leq \rho_{\tilde{N}}
\end{equation}
\end{lemma}

Proofs of lemmas can be found in the supplementary material accompanying this paper.

Note that each discontinuity point is calculated in the same way as the largest one, but with successive removals of the smallest principal components from our dataset. With all of the discontinuities defined, we can calculate the maximum mutual information for the case when our Frobenius norm constraint is contained in any of the corresponding intervals of continuity. 


\begin{theorem}\label{thm:lingaussmmi_small}
Take all of the assumptions from Theorem \ref{thm:lingaussmmi} except for the assumption that $F \geq \rho_{\tilde{N}}$ and take all definitions from lemma \ref{lemma:ordered_rho}. Let $\rho_{\tilde{N}-K+1} \geq F \geq \rho_{\tilde{N}-K}$. Then $\text{MMI}(X;Z)$ is given by:
\begin{align}\label{eqn:lingaussmmi_small}
    \frac{\tilde{N}-K}{2}log \left(\frac{F + \sigma^2Tr(\Sigma_{x,\tilde{N}-K}^{-1})}{\sigma^2 (\tilde{N}-K)} \right) + \frac{1}{2}log~|\Sigma_{x,\tilde{N}-K}|
\end{align}
\end{theorem}

\begin{proof}
We can follow the proof of Theorem \ref{thm:lingaussmmi} up until the optimization problem given by (\ref{prob:MMIgl_eigform_simp}), whose solution will now be different because $F$ has changed. Again, we need to find a solution consistent with $\tilde{\lambda}_i = \text{max}\left(0, \mu^*(F) - \frac{\sigma^2}{\lambda^x_{i}}\right)$ [optimality] and $\sum_i \text{max}\left(0, \mu^*(F) -  \frac{\sigma^2}{\lambda^x_{i}}\right) = F$ [consistency]. We claim that our assumption, ${\rho_{\tilde{N}-K+1} > F \geq \rho_{\tilde{N}-K}}$ yields a consistent solution in which the $\tilde{\lambda}_i$ are zero for ${i > \tilde{N} - K}$ and nonzero otherwise. Under such a solution, the consistency equation yields $\mu^*(F) = \frac{F + \sigma^2Tr(\Sigma_{x,\tilde{N}-K}^{-1})}{\tilde{N} - K}$. We will show that this coincides with the optimality equation. First, define $l(i) \triangleq \frac{F + \sigma^2Tr(\Sigma_{x,\tilde{N}-K}^{-1})}{\tilde{N} - K} - \frac{\sigma^2}{\lambda^x_{i}}$. Then, if ${i \leq \tilde{N}-K}$,  we have:
\begin{align}
    l(i) \geq \frac{F + \sigma^2Tr(\Sigma_{x,\tilde{N}-K}^{-1})}{\tilde{N} - K} - \frac{\sigma^2}{\lambda^x_{\tilde{N} - K}}
    = \frac{1}{\tilde{N} - K}\left(F - \rho_{\tilde{N} - K} \right) \geq 0
\end{align}
On the other hand, if ${i > \tilde{N}-K}$, then:
\begin{align}
    l(i) \leq \frac{F + \sigma^2Tr(\Sigma_{x,\tilde{N}-K}^{-1})}{\tilde{N} - K} - \frac{\sigma^2}{\lambda^x_{\tilde{N}-K+1}}
    &= \frac{F + \sigma^2Tr(\Sigma_{x,\tilde{N}-K+1}^{-1}) - \frac{\sigma^2}{\lambda^x_{\tilde{N}-K+1}}}{\tilde{N} - K} - \frac{\sigma^2}{\lambda^x_{\tilde{N}-K+1}} \notag \\
    &= \frac{1}{\tilde{N}-K}\left(F - \rho_{\tilde{N}-K+1} \right) < 0
\end{align}
And so optimality is achieved. Under this solution, the objective function value is given by:
\begin{equation}
    \sum_{i=\tilde{N}-K+1}^{\tilde{N}} log\left(\frac{\sigma^2}{\lambda^x_i}\right) + (\tilde{N} - K) log \left(\frac{F + \sigma^2Tr(\Sigma_{x,\tilde{N}-K}^{-1})}{\tilde{N} - K}\right)
\end{equation}
Thus, in all, we have that MMI$(X;Z)$ is given by:
\begin{align}
    \frac{\tilde{N}-K}{2}log \left(\frac{F + \sigma^2Tr(\Sigma_{x,\tilde{N}-K}^{-1})}{\sigma^2 (\tilde{N}-K)} \right) +  \frac{1}{2}log\left(\frac{|\Sigma_{x,\tilde{N}}|}{\prod_{i=\tilde{N}-K+1}^{\tilde{N}} \lambda^x_i}\right)
\end{align}
Where the factors of $\frac{1}{2}$ and the term $log (\frac{1}{\sigma^2}| \Sigma_{x, \tilde{N}}|)$ have come from equation (\ref{eqn:opt_fullform}). Finally, the only factors remaining in post-cancellation of the second term are the eigenvalues of $\Sigma_{x,\tilde{N}}$ with indices smaller than or equal to $\tilde{N}-K$, transforming that term into what is presented in equation (\ref{eqn:lingaussmmi_small}). This completes the proof.  
\end{proof}

Theorem \ref{thm:lingaussmmi_small} is a straightforward generalization of Theorem \ref{thm:lingaussmmi}. The only additional insight is that the Frobenius norm constraint $F$ acts to remove the smallest principal components from our maximum mutual information calculation. This role is similar to that of the hidden dimension. 

The MMI calculations in Theorems \ref{thm:lingaussmmi} and \ref{thm:lingaussmmi_small} will be seen to be very important - nearly every other case is a generalization of these two theorems! We've plotted some sample MMI curves as a function of the Frobenius norm $F$ for this family of architectures in Figure \ref{fig:MMI_curves}.

\begin{figure}[t]
    \centering
    \includegraphics[clip, trim=2.75cm 5cm 4cm 4.5cm, width=0.45\linewidth]{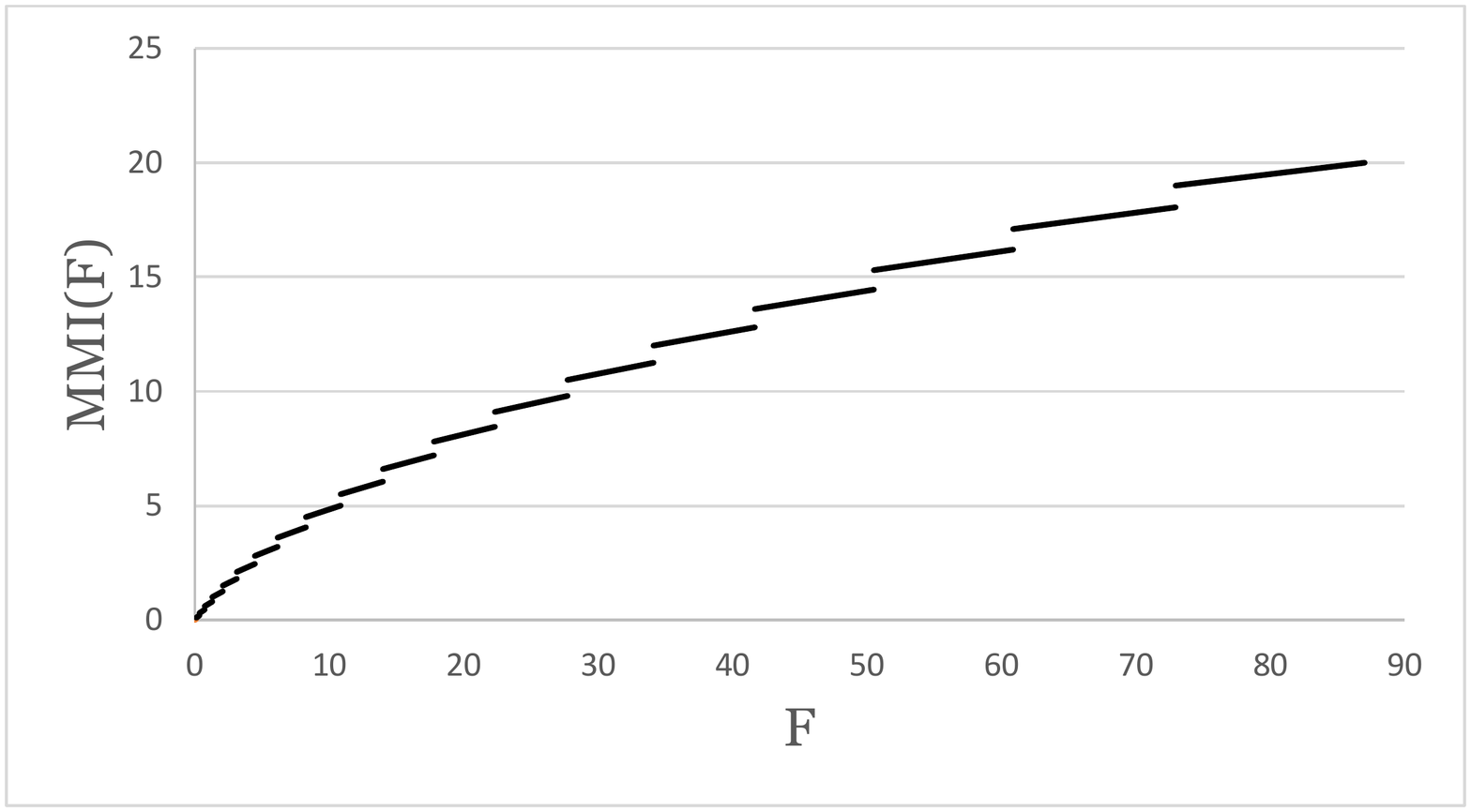}
    \includegraphics[clip, trim=2.5cm 9.75cm 3.75cm 10cm, width=0.45\linewidth]{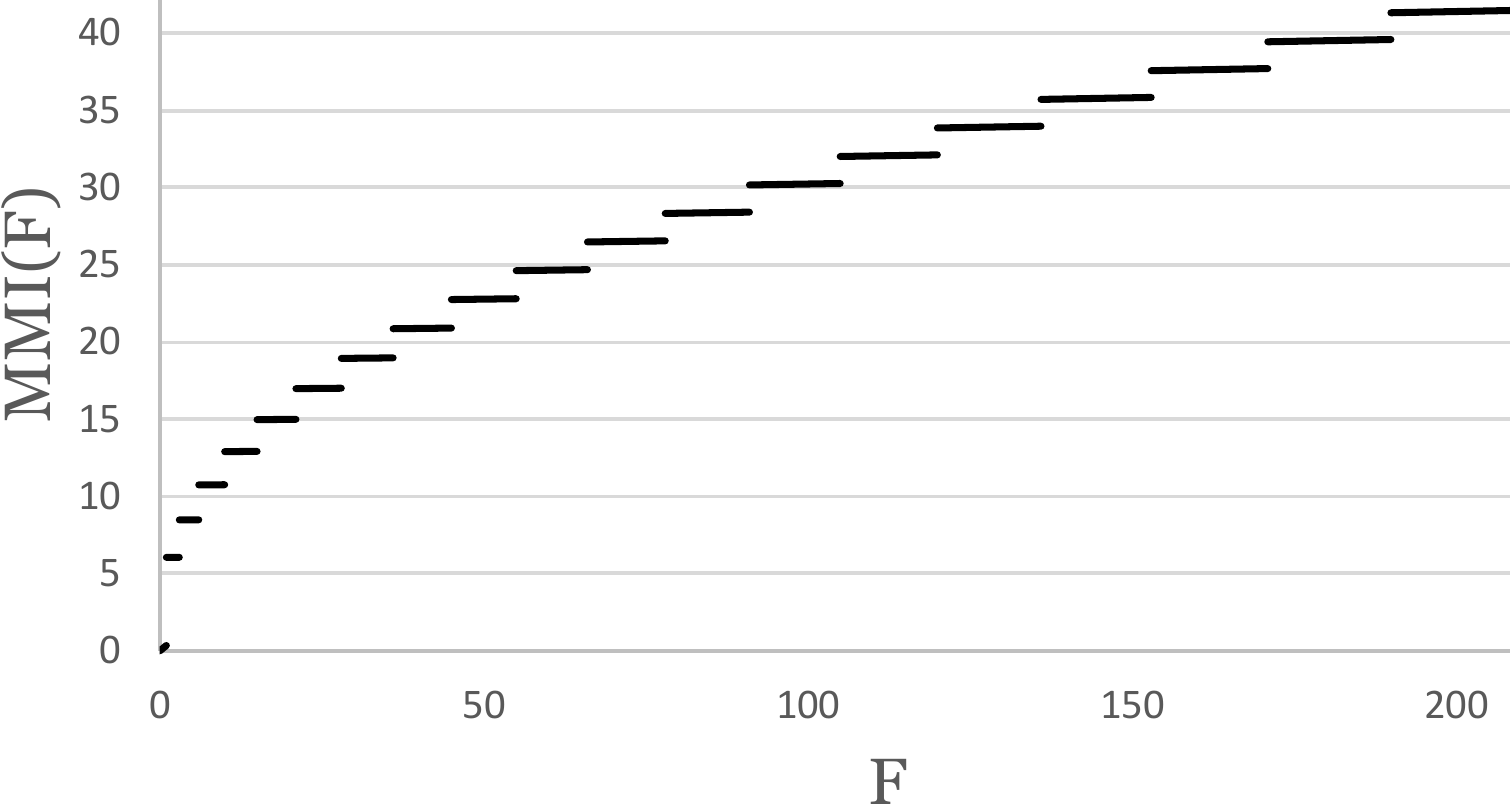}
    \caption{MMI(F) for a fully connected single layer architecture with 50 hidden units and $\sigma^2=1$ on a $100$ dimensional input dataset whose eigenvalues are modeled via [with indices starting at $1$] (Left) ${\lambda_i = e^{-0.1(i-1)}}$ (Right) ${\lambda_i = \frac{1}{i}}$.}
    \label{fig:MMI_curves}
\end{figure}

\subsection{Convolutional Case}
We will now discuss the case where our fully connected layer is replaced with a convolutional layer. We will keep the Gaussian data assumption and the linear activation assumption. For notational simplicity, we will also assume that our convolutional calculations involve non-overlapping strides, and that we have translation invariant statistics in our dataset. These assumptions can be dropped without effecting the insight that follows.


\begin{theorem}\label{thm:linconvmmi}
Let $\Sigma_x$ be a positive definite matrix and let $\sigma^2 > 0$.  Let $N_0$, $N_B$, and $N_f$ be natural numbers such that $N_B$ divides $N_0$, Let $X \sim \mc{N}(0; \Sigma_x)$, $X \in \mathbb{R}^{N_0}$, $Z|X,W,b \sim \mc{N}(X \circledast W+b; ~\sigma^2 Id_{N_1})$, $Z \in \mathbb{R}^{N_0/N_B}$
Where ${W \in \mathbb{R}^{N_f \times N_B}}$ and $X \circledast W = \begin{bmatrix} \tilde{X}_{1}^TW^T & \tilde{X}_{2}^TW^T & \cdots & \tilde{X}_{N_0/N_B}^TW^T \end{bmatrix}^T$ with $\tilde{X}_j$ denoting the slice of $X$ on indices ${(j-1)N_B +1}$ through $ jN_B$. Thus $ X \circledast W$ is a convolution applied to a vectorized input with non-overlapping stride and $N_f$ filters. Suppose $\Sigma_x$ is block diagonal with identical blocks (translation invariant statistics) given by the $N_B\times N_B$ matrix $\Sigma_{\tilde{x}}$. Let ${M_{LFC}(F; \Sigma_{x}, N_0, N_1)}$ denote the maximum mutual information of the linear fully connected network given by Theorems \ref{thm:lingaussmmi} and \ref{thm:lingaussmmi_small}. Let $F>0$ be fixed and let $\text{MMI}(X;Z) \triangleq \underset{Tr(W^TW) \leq F}{\text{ sup }} I_{W,b}(X;Z)$ Then $\text{MMI}(X;Z) = \frac{N_0}{N_B}M_{LFC}(F; \Sigma_{\tilde{x}}, N_B, N_f)$.
\end{theorem} 

\begin{proof}
We can view the output of the convolution as a matrix product of $\tilde{W}$ and $X$ where $\tilde{W}$ is a block diagonal matrix in which every nonzero block is the matrix $W$. We can then follow the proof of Theorem \ref{thm:lingaussmmi} until the matrix determinant lemma step with $\tilde{W}$ in place of $W$. From here, we can further factor as follows by noting that $\tilde{W}^T\tilde{W}$ and $\Sigma_x$ are both block diagonal:
\begin{align}
    log ~|\sigma^2 \Sigma_x^{-1} + \tilde{W}^T\tilde{W}| =\frac{N_0}{N_B}log~|\sigma^2\Sigma_{\tilde{x}}^{-1} + W^TW|,~~
    log~\left|\frac{\Sigma_x}{\sigma^2} \right| = \frac{N_0}{N_B}log~\left|\frac{\Sigma_{\tilde{x}}}{\sigma^2}\right|
\end{align}
from which we are back to the original optimization problem in the proof of Theorems \ref{thm:lingaussmmi} and \ref{thm:lingaussmmi_small}, just multiplied by a factor of $N_0/N_B$. 
\end{proof}

We see that each convolution operation acts as a parallel calculation, and are therefore summed in the MMI calculation. Furthermore, each convolution operation has an MMI bottleneck of $\min \{N_B, N_f \}$ which is the minimum of the (vectorized) block-size and the number of output channels. Note that the principal components used in the calculation correspond only to features that occur in a single convolution operation.


\section{Single Layer Relu Networks}
We will now move on to studying what happens to the MMI values when we place relu activations on the hidden layers. The answer is quite nice: nothing changes at all! Thus we can take all of the insight from the previous sections and apply them to relu networks. Unfortunately, it takes quite a bit of setup and rigor to prove this fact. 

We can provide some insight into the rigor that follows before diving into the proofs. Essentially, we will show that the mutual information between $X$ and $Z$ in a relu network is always bounded above by that of a corresponding linear network. However, we will be able to construct a sequence of relu networks whose marginal distributions on $Z$ converge (weakly) to that of the maximum mutual information solution of the linear architectures. All we will need to do then is have some notion of continuity of mutual information with respect to these marginals and we will have our proof. This penultimate step is taken care of primarily in Lemma \ref{lemma:relu_info_diff}.


\begin{lemma}\label{lemma:Z_relu_entropy}
Take all of the assumptions of either Theorem \ref{thm:lingaussmmi}, Theorem \ref{thm:lingaussmmi_small}, or Theorem \ref{thm:linconvmmi}. Let $W, b$ be fixed and define ${Z_{relu}}$ through a new model given by $Z_{relu}|X,W,b \sim \mc{N}(\text{relu}(WX+b); ~\sigma^2 Id_{N_1})$ for the fully connected case, or $Z_{relu}|X,W,b \sim \mc{N}(\text{relu}(X \circledast W+b); ~\sigma^2 Id_{N_1})$ for the convolutional case. Let $\Sigma_Z$ denote the covariance matrix of $Z$ as defined in Theorem \ref{thm:lingaussmmi} and let $\Sigma_{Z_{relu}}$ denote  the covariance matrix of $Z_{relu}$. 
Then:
\begin{equation}
    H(Z_{relu}) \leq H(Z)
\end{equation}
\end{lemma}


\begin{lemma}\label{lemma:relu_info_diff}
Take all of the assumptions of either Theorem \ref{thm:lingaussmmi} or Theorem \ref{thm:lingaussmmi_small} and all definitions from lemma \ref{lemma:Z_relu_entropy}. Denote the marginal probability laws of $Z$ and ${Z_{relu}}$ as $\mathbb{P}$ and $\hat{\mathbb{P}}$ with densities denoted $p(z)$ and $\hat{p}(z)$. Let $\mc{B}(Z)$ be the set of Borel measurable sets on $\mc{Z}$, and let $\delta$ denote the total variation distance between $\mathbb{P}$ and $\hat{\mathbb{P}}$. That is, $\delta = \sup_{A \in \mc{B}(Z)} |\mathbb{P}(A) - \hat{\mathbb{P}}(A)| = \frac{1}{2}\int |p(z) - \hat{p}(z)| dz$. Finally, let $h_2(\cdot)$ denote the binary entropy function. Let $\epsilon > 0$. Then for all ${W,b}$ such that $\delta < \frac{1}{e}$, there exists a non-negative function $g(\delta(W,b))$ which is continuous in $\delta$ from the right at $\delta=0$, has $g(0) = 0$, and :
\begin{equation}
    |I_{W,b}(X;Z) - I_{W,b}(X;Z_{relu})| \leq g(\delta(W,b))
\end{equation}
\end{lemma}

Proofs of lemmas can be found in the supplementary material accompanying this paper.

\begin{theorem}\label{thm:relummi}
Take all of the assumptions of either Theorem \ref{thm:lingaussmmi} or Theorem \ref{thm:lingaussmmi_small}, and take all definitions from lemmas \ref{lemma:Z_relu_entropy} and \ref{lemma:relu_info_diff}. Then the results of Theorem \ref{thm:lingaussmmi} and Theorem \ref{thm:lingaussmmi_small} hold for $\text{MMI}(X; Z_{relu})$. That is:
\begin{equation}
    \text{MMI}(X; Z_{relu}) = \text{MMI}(X; Z)
\end{equation}
\end{theorem}
\begin{proof}
First, we note that, for all $W,b$, we have:
\begin{align}
    I_{W,b}(X;Z_{relu}) = H_{W,b}(Z_{relu}) - H(Z_{relu}|X) &= H_{W,b}(Z_{relu}) - H(Z|X) \notag \\
    &\leq H_{W, b}(Z) - H(Z|X) = I_{W,b}(X,Z)
\end{align}
where the inequality follows from lemma \ref{lemma:Z_relu_entropy}. It follows immediately that $\text{MMI}(X;Z_{relu}) \leq \text{MMI}(X;Z)$. We now show that $\text{MMI}(X;Z)$ is an achievable value for $I_{W,b}(X;Z_{relu})$ given the constraint ${Tr(W^TW) \leq F}$. 
Fix ${\frac{1}{e} > \epsilon > 0}$. Then given any value of $F$, we can set $b$ large enough in each dimension such that ${\hat{\mathbb{P}}({\cup_{i=1}^{N_1} \{z | z_i  < 0 \}})}$ is less than $\epsilon$ for all $W$ satisfying ${Tr(W^TW) \leq F}$. When this is the case, the total variation between $\mathbb{P}$ and $\hat{\mathbb{P}}$ is also bounded above by $\epsilon$. Then by lemma \ref{lemma:relu_info_diff}, there exists $b^* \in \mathbb{R}^{N_1}$ such that:
\begin{equation}
    |I_{W,b^*}(X;Z) - I_{W,b^*}(X;Z_{relu})| \leq g(\epsilon)
\end{equation}
for all $W$ satisfying this constraint, and where the right hand side of this inequality is a continuous function of $\epsilon$ and satisfies $\lim_{\epsilon \to 0^+}g(\epsilon) = g(0) = 0$. Thus we can achieve ${I_{W,b^*}(X;Z_{relu}) \geq I_{W,b^*}(X;Z) - g(\epsilon)}$ for all $W$ satisfying the constraint. Since $g \geq 0$ can be made arbitrary small via continuity, we can achieve:
\begin{equation}\label{eqn:relu_achieve}
    I_{W,b^*}(X;Z_{relu}) \geq I_{W,b^*}(X;Z)~\forall ~W~ \text{s.t.}~ Tr(W^TW)\leq F
\end{equation}
Inputting the MMI achieving matrix from Theorems \ref{thm:lingaussmmi} and \ref{thm:lingaussmmi_small} into (\ref{eqn:relu_achieve}) yields the result. 
\end{proof}

We see that all of the insights that we obtained for linear activated networks hold for relu activated networks as well. 


\section{Single Layer Fully Connected Networks with Bijective Activation Functions}
We will now move on to deriving the MMI for one final family of architectures - single layer networks with bijective activation functions. This family includes sigmoidal activations, tanh activations, selu activations, and much more.  Conveniently, we once again find ourselves looking back to the linear case for its calculation. 

\begin{theorem}\label{lem:bijective_activation}
Let $A$ denote the pre-activated representation variable of a single layer neural network. Let $\phi$ be a bijective activation whose log-derivative has finite expectation, and let the representation $Z$ be given by ${Z=\phi(A)}$. Then:
\begin{equation}
    I(X;Z) = I(X;A)
\end{equation}

\end{theorem}
\begin{proof}
    This follows immediately from the following set of equalities:
    \begin{align}
        I(X;Z) = H(Z) - H(Z|X), ~H(Z) = H(A) + \mathbb{E}_{\mathbb{P}_A} \left[ log \left|\frac{d\phi}{da}\right| \right] \notag \\
        H(Z|X) = H(A|X) + \mathbb{E}_{\mathbb{P}_{X,A}} \left[ log \left|\frac{d\phi}{da} \right| \right], ~
        \mathbb{E}_{\mathbb{P}_{X,A}} \left[ log \left|\frac{d\phi}{da} \right| \right] =  \mathbb{E}_{\mathbb{P}_A} \left[ log \left|\frac{d\phi}{da}\right| \right] \notag \\
        I(X;A) = H(A) - H(A|X)
    \end{align}
    where $\left| \frac{d\phi}{da} \right|$ is the determinant of the Jacobian matrix of $\phi$. 
\end{proof}

From this theorem, we can immediately see that, if we take the linear case and place our noise injection on the pre-activated variable, and then use ${Z=\phi(A)}$, we will have the same MMI as we did in the linearly activated case.


\section{Multilayer Fully Connected Linear,  Relu, and Bijective Networks}
We finally move on to the multi-layer case for linear and relu fully connected networks. 


\begin{theorem}
Take all assumptions and definitions from the previous theorems corresponding to a fully connected network (linear or relu), but assume that we are using a $K$ layer neural network instead of a single layer network, with the noise placed on the $K^{th}$ layer. Let $N_0, N_1, \cdots, N_K$ denote the number of hidden units in each layer. Redefine $\tilde{N}$ to ${\tilde{N} \triangleq \min(N_0, N_1, \cdots, N_K)}$. Then the results of those previous theorems hold. 
\end{theorem}

\begin{proof}
In the linear case, we can take the proof of theorem \ref{thm:lingaussmmi} by replacing $W$ with ${W_K \cdots W_2 W_1}$ (the biases have no effect on the mutual information). We will only need to note that the corresponding inner-product matrix, ${W_1^TW_2^T \cdots W_K^TW_K\cdots W_2W_1}$ has rank $\tilde{N}$ (as redefined in this theorem's hypothesis). In the relu case we can follow the exact sequence of steps that were performed in the single layer case, noting that we can get the marginal total variation ${\delta < \epsilon}$ (for any fixed $\epsilon>0$) by fixing each bias to be large enough such that sufficiently small amounts of marginal probability are contained in the saturated regions of each layer. 
\end{proof}

We see that all of our previous insights for these families of fully connected networks hold. However, a new MMI bottleneck can be identified: $\min \{N_0, N_1, \cdots, N_K\}$ - the dimension of the smallest layer of the network. 


\section{Conclusion}
We have rigorously derived the Maximum Mutual Information for a large number of neural network architectures, and provided several insights along the way. Nearly every case generalizes from Theorems \ref{thm:lingaussmmi} and \ref{thm:lingaussmmi_small}. All of the studied fully connected single layer architectures have the exact same MMI expressions as those studied in Theorems \ref{thm:lingaussmmi} and \ref{thm:lingaussmmi_small}, and single layer convolutional architectures require only a small adjustment from those calculations (Theorem \ref{thm:linconvmmi}). Multi-layer networks generalize from these as well, but with the primary architectural parameter being given by the smallest hidden dimension in the network.  Thus great care should be given to the design of this layer when attempting to control the network's mutual information.

While we have not provided MMI calculations for every existing setup (doing so in one paper would be nearly impossible given the ever-expanding set of neural architectures in existence), we hope that these calculations and insights provide enough detail to approximate any architecture  the reader may be interested in studying, and that the proofs provided are generalizable enough for when an approximation is not enough. More architectures are the subject of future work. Of particular future interest are the inherently lossy dropout mechanisms, pooling strategies, normalization, and skip layers. 

\bibliographystyle{ieeetran}
\bibliography{refs}

\title{On the Maximum Mutual Information Capacity of Neural Architectures - Appendix}
\maketitle

\section{Proofs of Lemmas}

We repeat the theorem statements (without their proofs) alongside these lemmas to minimize the amount of back and fourth jumps the reader must perform to read them, as we refer to the theorems several times in the lemmas.

\begin{theorem}\label{thm:lingaussmmi}
Let $\Sigma_x$ be a positive definite matrix and let $\sigma^2 > 0$. Let $N_0$ and $N_1$ be natural numbers representing the input and hidden dimensions. Let $\mc{N}(\mu; A)$ denote the Gaussian distribution with mean $\mu$ and covariance matrix $A$. Let $ X \sim \mc{N}(0; \Sigma_x)$, $X \in \mathbb{R}^{N_0}$, $Z|X,W,b \sim \mc{N}(WX+b; ~\sigma^2 Id_{N_1})$, $Z \in \mathbb{R}^{N_1}$, where $W \in \mathbb{R}^{N_1 \times N_0}$, $Id_{N_1}$ is the identity matrix in $N_1$ dimensions, and $b \in \mathbb{R}^{N_1}$ is the bias vector. Let ${\tilde{N}=\min(N_0, N_1)}$. Let $\Sigma_{x,\tilde{N}}$ denote $\tilde{N} \times \tilde{N}$ diagonal matrix containing the $\tilde{N}$ largest eigenvalues of $\Sigma_x$. Let $\lambda^x_{\tilde{N}}$ denote the smallest eigenvalue of $\Sigma_{x,\Tilde{N}}$, and let $\rho_{\tilde{N}} \triangleq \sigma^2 \left(\frac{\tilde{N}}{\lambda^x_{\tilde{N}}} - Tr(\Sigma_{x,\tilde{N}}^{-1})\right)$. Let $F \geq \rho_{\tilde{N}}$, $I_{W,b}(X;Z) \triangleq \mathbb{E}_{\sim p(x,z|W, b)}\left[log \frac{p(z|x,W, b)}{p(z|W, b)} \right]$, and define $\text{MMI}(X;Z) \triangleq \underset{Tr(W^TW) \leq F}{\text{ sup }} I_{W,b}(X;Z)$. Then: 
\begin{equation}
    \text{MMI}(X;Z) = \frac{\tilde{N}}{2}log \left(\frac{F + \sigma^2Tr(\Sigma_{x,\tilde{N}}^{-1})}{\sigma^2 \tilde{N}} \right) + \frac{1}{2}log~|\Sigma_{x,\tilde{N}}|
\end{equation}
\end{theorem}

\begin{lemma}\label{lemma:ordered_rho}
Take all of the assumptions from Theorem \ref{thm:lingaussmmi} except for the assumption that ${F \geq \rho}$. Let the $k^{th}$ largest eigenvalue of $\Sigma_{x}$ be denoted by $\lambda^x_k$. Let $K$ be a natural number, ${K < \tilde{N}}$, and let $\Sigma_{x,\tilde{N}-K}$ denote the ${(\tilde{N}-K) \times (\tilde{N}-K)}$ diagonal matrix containing the $\tilde{N}-K$ largest eigenvalues of $\Sigma_x$. Now, let $\rho_{\tilde{N}-K} \triangleq \sigma^2 \left(\frac{\tilde{N}-K}{\lambda^x_{\tilde{N}-K}} - Tr\left(\Sigma_{x,\tilde{N}-K}^{-1}\right)\right)$. Then:
\begin{equation}
    0 = \rho_{1} \leq \cdots \leq \rho_{\tilde{N}-K} \leq \rho_{\tilde{N}-K+1} \leq \cdots \leq \rho_{\tilde{N} - 1} \leq \rho_{\tilde{N}}
\end{equation}
\end{lemma}
\begin{proof}
First, $\rho_1$ is zero since: 
\begin{equation}
    \rho_1 = \sigma^2\left(\frac{1}{\lambda^x_{1}} - \frac{1}{\lambda^x_{1}} \right)
\end{equation}

Next, we note that the difference ${\rho_{\tilde{N}-K+1} - \rho_{\tilde{N}-K}}$ is given by:
\begin{align}
     \quad \sigma^2\left(\frac{\tilde{N} - K + 1}{\lambda^x_{\tilde{N}-K+1}} - \frac{\tilde{N} - K}{\lambda^x_{\tilde{N}-K}} - \frac{1}{\lambda_{\tilde{N}-K+1}^x} \right)
    &=  \sigma^2\left(\frac{\tilde{N} - K}{\lambda^x_{\tilde{N}-K+1}} - \frac{\tilde{N} - K}{\lambda^x_{\tilde{N}-K}} \right) \notag \\
    &\geq \sigma^2\left(\frac{\tilde{N} - K}{\lambda^x_{\tilde{N}-K}} - \frac{\tilde{N} - K}{\lambda^x_{\tilde{N}-K}} \right) =  0
\end{align}
completing the proof.
\end{proof}

\begin{theorem}\label{thm:lingaussmmi_small}
Take all of the assumptions from Theorem \ref{thm:lingaussmmi} except for the assumption that $F \geq \rho_{\tilde{N}}$ and take all definitions from lemma \ref{lemma:ordered_rho}. Let $\rho_{\tilde{N}-K+1} \geq F \geq \rho_{\tilde{N}-K}$. Then $\text{MMI}(X;Z)$ is given by:
\begin{align}\label{eqn:lingaussmmi_small}
    \frac{\tilde{N}-K}{2}log \left(\frac{F + \sigma^2Tr(\Sigma_{x,\tilde{N}-K}^{-1})}{\sigma^2 (\tilde{N}-K)} \right) + \frac{1}{2}log~|\Sigma_{x,\tilde{N}-K}|
\end{align}
\end{theorem}

\begin{theorem}\label{thm:linconvmmi}
Let $\Sigma_x$ be a positive definite matrix and let $\sigma^2 > 0$.  Let $N_0$, $N_B$, and $N_f$ be natural numbers such that $N_B$ divides $N_0$, Let $X \sim \mc{N}(0; \Sigma_x)$, $X \in \mathbb{R}^{N_0}$, $Z|X,W,b \sim \mc{N}(X \circledast W+b; ~\sigma^2 Id_{N_1})$, $Z \in \mathbb{R}^{N_0/N_B}$
Where ${W \in \mathbb{R}^{N_f \times N_B}}$ and $X \circledast W = \begin{bmatrix} \tilde{X}_{1}^TW^T  \tilde{X}_{2}^TW^T  \cdots & \tilde{X}_{N_0/N_B}^TW^T \end{bmatrix}^T$ with $\tilde{X}_j$ denoting the slice of $X$ on indices ${(j-1)N_B +1}$ through $ jN_B$. Thus $ X \circledast W$ is a convolution applied to a vectorized input with non-overlapping stride and $N_f$ filters. Suppose $\Sigma_x$ is block diagonal with identical blocks (translation invariant statistics) given by the $N_B\times N_B$ matrix $\Sigma_{\tilde{x}}$. Let ${M_{LFC}(F; \Sigma_{x}, N_0, N_1)}$ denote the maximum mutual information of the linear fully connected network given by Theorems \ref{thm:lingaussmmi} and \ref{thm:lingaussmmi_small}. Let $F>0$ be fixed and let $\text{MMI}(X;Z) \triangleq \underset{Tr(W^TW) \leq F}{\text{ sup }} I_{W,b}(X;Z)$. Then $\text{MMI}(X;Z) = \frac{N_0}{N_B}M_{LFC}(F; \Sigma_{\tilde{x}}, N_B, N_f)$.
\end{theorem} 

\begin{lemma}\label{lemma:Z_relu_entropy}
Take all of the assumptions of either Theorem \ref{thm:lingaussmmi}, Theorem \ref{thm:lingaussmmi_small}, or Theorem \ref{thm:linconvmmi}. Let $W, b$ be fixed and define ${Z_{relu}}$ through a new model given by $Z_{relu}|X,W,b \sim \mc{N}(\text{relu}(WX+b); ~\sigma^2 Id_{N_1})$ for the fully connected case, or $Z_{relu}|X,W,b \sim \mc{N}(\text{relu}(X \circledast W+b); ~\sigma^2 Id_{N_1})$ for the convolutional case. Let $\Sigma_Z$ denote the covariance matrix of $Z$ as defined in Theorem \ref{thm:lingaussmmi} and let $\Sigma_{Z_{relu}}$ denote  the covariance matrix of $Z_{relu}$. 
Then:
\begin{equation}
    H(Z_{relu}) \leq H(Z)
\end{equation}
\end{lemma}
\begin{proof}
Let $\eta$ denote a multivariate Gaussian with covariance $\sigma^2Id_{N_1}$. We further denote ${S \triangleq WX+b}$. Then:
\begin{align}
    Z = S + \eta,  ~~ Z_{relu} = relu(S) + \eta \notag \\
\end{align}
Then we have:
\begin{align}
    H(Z, Z_{relu}) &= H(Z) + H(Z_{relu} | Z) = H(Z_{relu}) + H(Z | Z_{relu})
\end{align}
It follows that:
\begin{align}
    H(Z) = H(Z_{relu}) + H(Z | Z_{relu}) - H(Z_{relu} | Z)
\end{align}
But $H(Z | Z_{relu}) \geq H(Z_{relu} | Z)$ since knowledge of $Z_{relu}$ tells us far less about $Z$ than knowledge of $Z$ tells us about $Z_{relu}$.

\end{proof}

The next lemma relies heavily on the concept of a probabilistic coupling. Thus we will first review this concept.
\begin{defn}
Given two probability models $\mathbb{P}_{\tilde{S}}$ and $\mathbb{Q}_S$ on a list of variables $S$, a \textbf{coupling} of these models is a pair of random variables $(\tilde{S}, \hat{S})$ with joint distribution $\gamma_{\tilde{S}, \hat{S}}$ such that the marginal distributions satisfy $\gamma_{\tilde{S}} = \mathbb{P}_{\tilde{S}}$ and $\gamma_{\hat{S}} = \mathbb{Q}_S$.    
\end{defn}

In particular, we will rely on the concept of a \textit{maximal coupling}, which is defined as follows:
\begin{defn}
Given two probability measures $\mathbb{P}$ and $\hat{\mathbb{P}}$ on a Euclidean space $\mc{X}$ with probability density functions $p$ and $\hat{p}$, the maximal coupling on this pair is defined as follows:

First, define the function ${m: \mc{X}\to [0,1]}$ through $m(a) := \int \text{min}\{p_{Y|X}(b|a), \hat{p}_{Y|X}(b|a) \}db$. Next, define a real number $\rho$ as $\rho := \int m d\mathbb{P}$ and define $J$ as a Bernoulli random variable with success probability $\rho$. Then define variables $U, V$ and $W$ with the following distributions: 
\begin{align}
p_{U} := \frac{m}{\rho}, ~
p_{V} := \frac{p - m}{1- \rho}, ~
p_{W} := \frac{\hat{p} - m}{1- \rho} 
\end{align}
Next define $(\tilde{X}, \hat{X})$ as functions of the above random variables as follows:
\begin{equation}
\begin{cases}
\tilde{X} = \hat{X} = U  & \text{if } J = 1 \\
\tilde{X} = V, ~~~~\hat{X} = W, & \text{if } J = 0
\end{cases}
\end{equation}

The pair ($\tilde{X}$, $\hat{X}$), and its distribution $\gamma$ form the maximal coupling. One can show that $1-\rho$ is the total variation distance between $p$ and $\hat{p}$.
\end{defn}

\begin{lemma}\label{lemma:relu_info_diff}
Take all of the assumptions of either Theorem \ref{thm:lingaussmmi} or Theorem \ref{thm:lingaussmmi_small} and all definitions from lemma \ref{lemma:Z_relu_entropy}. Denote the marginal probability laws of $Z$ and ${Z_{relu}}$ as $\mathbb{P}$ and $\hat{\mathbb{P}}$ with densities denoted $p(z)$ and $\hat{p}(z)$. Let $\mc{B}(Z)$ be the set of Borel measurable sets on $\mc{Z}$, and let $\delta$ denote the total variation distance between $\mathbb{P}$ and $\hat{\mathbb{P}}$. That is, $\delta = \sup_{A \in \mc{B}(Z)} |\mathbb{P}(A) - \hat{\mathbb{P}}(A)| = \frac{1}{2}\int |p(z) - \hat{p}(z)| dz$. Finally, let $h_2(\cdot)$ denote the binary entropy function. Let $\epsilon > 0$. Then for all ${W,b}$ such that $\delta < \frac{1}{e}$, there exists a non-negative function $g(\delta(W,b))$ which is continuous in $\delta$ from the right at $\delta=0$, has $g(0) = 0$, and :
\begin{equation}
    |I_{W,b}(X;Z) - I_{W,b}(X;Z_{relu})| \leq g(\delta(W,b))
\end{equation}
\end{lemma}
\begin{proof}
Let ($\tilde{Z}$, $\hat{Z}$, $\gamma$) denote the maximal coupling between $\mathbb{P}$ and $\hat{\mathbb{P}}$ on $(\tilde{Z}, \hat{Z})$. Then, as ${H(\tilde{Z}|X) = H(\hat{Z}|X)}$ (all of the uncertainty under these conditionals comes from the noise variable $\eta$), we have:
\begin{align}
    |I(X;Z) - I(X;\hat{Z})| &= |H(\tilde{Z}) - H(\hat{Z})|
\end{align}
We can decompose these terms as:
\begin{align}
    H(\tilde{Z}) = H(\tilde{Z}|J) + H(J) - H(J|\tilde{Z}), ~
    H(\hat{Z}) = H(\hat{Z}|J) + H(J) - H(J|\hat{Z})
\end{align}
The $H(J)$ terms cancel in the subtraction. Furthermore, both $H(J|\tilde{Z})$ and $H(J|\hat{Z})$ are bounded by $H(J) = h_2(\delta)$. Thus, by an application of the triangle inequality, we have:
\begin{equation}
    |I(X;Z) - I(X;\hat{Z})| \leq |H(\tilde{Z}|J) - H(\hat{Z}|J)| + 2h_2(\delta)
\end{equation}
We can further decompose the remaining terms as:
\begin{align}
    H(\tilde{Z}|J) = (1-\delta)H(U) + \delta H(V), ~
    H(\hat{Z}|J) = (1-\delta)H(U) + \delta H(W)
\end{align}
leaving us with: 
\begin{equation}\label{eqn:info_eqn}
    |I(X;Z) - I(X;\hat{Z})| \leq \delta|H(V) - H(W)| + 2h_2(\delta)
\end{equation}

Now, observe that the absolute entropy difference, $|H(V) - H(W)|$, expands to:

\begin{align}
    \left| \frac{1}{\delta}\int \left\{(p(z) - m(z))log\left(\frac{p(z)-m(z)}{\delta}\right) - (\hat{p}(z) - m(z))log\left(\frac{\hat{p}(z)-m(z)}{\delta}\right) \right\}dz \right| \notag
\end{align}

Let $A$ denote the set of points in $Range(Z)$ such that $p(z) \geq \hat{p}(z)$. Then the second term in the expression becomes zero over $A$ and the first term becomes zero over $A^c$. We are then left with:
\begin{equation}
    \frac{1}{\delta} \left| \int_A (p(z) - \hat{p}(z))log\left(\frac{p(z)-\hat{p}(z)}{\delta}\right)dz - \int_{A^c}(\hat{p}(z) - p(z))log\left(\frac{\hat{p}(z)-p(z)}{\delta}\right) dz \right| \notag
\end{equation}
which can equivalently be written as:
\begin{align}
    \frac{1}{\delta}\left| \int (p(z) - \hat{p}(z))log~|p(z)-\hat{p}(z)| dz \right | \leq \frac{1}{\delta} \int |p(z) - \hat{p}(z)||log~|p(z)-\hat{p}(z)|| dz
\end{align}
Now, split the domain of this last integral into three disjoint subregions $Q_1$, $Q_2$, and $Q_3$ where $Q_1$ is defined by the condition that $|p(z)-\hat{p}(z)| \leq 2 \delta$, $Q_2$ is given by the condition that $2\delta < |p(z)-\hat{p}(z)| \leq 1$, and $Q_3$ is given by the condition that $|p(z)-\hat{p}(z)| > 1$. We then have by the definition of total variation that $\frac{1}{2}\int |p(z) - \hat{p}(z)| = \delta$, so it must be the case that:
\begin{equation}
    \int_{Q_2} |p(z) - \hat{p}(z)| \leq 2\delta -  \int_{Q_1 \cup Q_3} |p(z) - \hat{p}(z)| \leq 2\delta
\end{equation}
Futhermore, in $Q_2$, $|log|p(z) - \hat{p}(z)||$ is decreasing in $|p(z) - \hat{p}(z)|$, so we can bound it above by $|log(2\delta)|$. We then have that:
\begin{align}
    \int_{Q_2} |p(z) - \hat{p}(z)||log~|p(z)-\hat{p}(z)|| dz  \leq |log~2\delta| \int_{Q_2} |p(z) - \hat{p}(z)| dz \leq 2\delta|log(2\delta)|
\end{align}
Also, in $Q_1$, if $\delta < \frac{1}{e}$, as was assumed in this Lemma's hypothesis, $|p(z) - \hat{p}(z)||log~|p(z)-\hat{p}(z)||$ is decreasing in $|p(z) - \hat{p}(z)|$. It can thus be bounded above by $2\delta|log(2\delta)|.$ Finally, the addition of the noise term $\eta$ in both models guarantees a uniform upper bound via Young's Convolution Inequality given by $|p(z) - \hat{p}(z)| < \frac{1}{(2\pi\sigma^2)^{N_1/2}}$ (independent of both $z$ and $\delta$). A derivation of this bound is provided at the end of this proof. Given that we are in $Q_3$, we then have an upper bound given by $M = max\{1, \frac{1}{(2\pi\sigma^2)^{N_1/2}}\}$.  Thus the $Q_3$ integral can be bounded as:
\begin{equation}
    \int_{Q_3}|p(z) - \hat{p}(z)||log~|p(z)-\hat{p}(z)|| dz \leq |log(M)|\int_{Q_3}|p(z) - \hat{p}(z)| \leq 2\delta|log(M)|
\end{equation}
Thus the full entropy difference expression can be bounded above by $\frac{4\delta |log(2\delta)| + 2\delta|log(M)|}{\delta}$. Plugging this back into equation (\ref{eqn:info_eqn}) gives us:
\begin{equation}
    |I(X;Z) - I(X;\hat{Z})| \leq 4\delta|log(2\delta)| + 2\delta|log(M)| + 2h_2(\delta)
\end{equation}
The term on the right hand side is the desired continuous function $g(\delta)$.

\vspace{1em}

The bound on $|p(z) - \hat{p}(z)|$ is derived from Young's Convolution Inequality. Letting $p_{r}$ be the density function of the random variable $Relu(WX+b)$ and $p_{\eta}$ that of $\eta$. Then $\hat{p}$ is the convolution of $p_r$ and $p_{\eta}$. We then have that $\|\hat{p}\|_{\infty} \leq \int p_r(z)dz \cdot \|p_{\eta}\|_{\infty}$ ($p=1, q=r=\infty$ in Young's inequality). The integral evaluates to one, and we are left with a bound of the supremum of $p_{\eta}$, which is given by its multivariate normal pdf evaluated at its mean. Since $p$ is bounded in exactly the same way, we have the desired result.

\end{proof}

\begin{theorem}\label{thm:relummi}
Take all of the assumptions of either Theorem \ref{thm:lingaussmmi} or Theorem \ref{thm:lingaussmmi_small}, and take all definitions from lemmas \ref{lemma:Z_relu_entropy} and \ref{lemma:relu_info_diff}. Then the results of Theorem \ref{thm:lingaussmmi} and Theorem \ref{thm:lingaussmmi_small} hold for $\text{MMI}(X; Z_{relu})$. That is:
\begin{equation}
    \text{MMI}(X; Z_{relu}) = \text{MMI}(X; Z)
\end{equation}
\end{theorem}

\begin{theorem}\label{lem:bijective_activation}
Let $A$ denote the pre-activated representation variable of a single layer neural network. Let $g$ be a bijective activation and let the representation $Z$ be given by ${Z=g(A)}$. Then:
\begin{equation}
    I(X;Z) = I(X;A)
\end{equation}

\end{theorem}

\begin{theorem}
Take all assumptions and definitions from the previous theorems corresponding to a fully connected network, but assume that we are using a $K$ layer neural network instead of a single layer network, with the noise placed on the $K^{th}$ layer. Let $N_0, N_1, \cdots, N_K$ denote the number of hidden units in each layer. Redefine $\tilde{N}$ to ${\tilde{N} \triangleq \min(N_0, N_1, \cdots, N_K)}$. Then the results of those previous theorems hold. 
\end{theorem}

\end{document}